\documentclass[11pt]{article}
\usepackage[letterpaper,total={5.5in,8in}]{geometry}
\usepackage{libertine}

\usepackage[utf8]{inputenc}
\usepackage[T1]{fontenc}
\usepackage{microtype}      
\usepackage{url}            
\usepackage{booktabs}
\usepackage{multirow}
\usepackage{comment}
\usepackage{siunitx}

\usepackage[dvipsnames]{xcolor}
\usepackage[colorlinks]{hyperref}

\usepackage{amsmath}
\usepackage{amssymb}
\usepackage{amsfonts}
\usepackage{mathtools}
\usepackage{bm}
\usepackage{bbm}
\usepackage{nicefrac}
\usepackage{centernot}

\usepackage{amsthm}
\theoremstyle{plain}
\newtheorem{proposition}{Proposition}

\usepackage{float}
\usepackage{graphicx}
\usepackage{wrapfig}
\usepackage{subcaption}
\graphicspath{{./figures/}}
\DeclareGraphicsExtensions{.pdf,.png,.jpeg,.jpg}

\usepackage[capitalize,sort&compress]{cleveref}

\def\X{{\mathcal X}}
\def\Y{{\mathcal Y}}
\def\A{{\mathcal A}}
\def\S{{\mathcal S}}

\def\E{{\mathbb E}}
\def\P{{\mathbb P}}
\def\R{{\mathbb R}}

\def\1{{\bm 1}}
\def\vlambda{{\bm \lambda}}

\def\ve{{\bm e}}

\def\l{{\ell}}

\def\hl{{\hat \l}}

\def\q{{\hat q}}
\def\hlambda{{\hat \lambda}}
\def\hvlambda{{\hat \vlambda}}

\def\I{{\bar I}}

\def\tvlambda{{\widetilde \vlambda}}

\def\cal{{\text{cal}}}

\def\opt{{\text{opt}}}
\def\st{\text{s.t.}}
\def\sem{{\text{sem}}}
\def\bsem{{\overline{\text{sem}}}}

\def\RCPS{{\text{RCPS}}}
\def\kRCPS{{K\text{-}\RCPS}}

\def\CRC{{\text{CRC}}}
\def\kCRC{{K\text{-}\CRC}}
\def\semCRC{{sem\text{-}\CRC}}
\def\bsemCRC{{\overline{sem}\text{-}\CRC}}

\newcommand{\at}[1]{^{(#1)}}
\newcommand{\smallsize}{\fontsize{8pt}{20pt}\selectfont}
\DeclareMathOperator*{\argmin}{\arg\min}

\usepackage[numbers,sort&compress]{natbib}
\title{\textbf{Conformal Risk Control for Semantic Uncertainty\\Quantification in Computed Tomography}}

\author{Jacopo Teneggi\\\texttt{jtenegg1@jhu.edu}}
\author{
    Jacopo~Teneggi$^{1,2}$\\
    \texttt{\small jtenegg1@jhu.edu}
    \and
    J Webster Stayman$^{3}$\\
    \texttt{\small web.stayman@jhu.edu}
    \and
    Jeremias~Sulam$^{1,3}$\\
    \texttt{\small jsulam1@jhu.edu}
}
\date{
    \centering
    \small
    \vspace{1em}
    $^1$Mathematical Institute for Data Science (MINDS), Johns Hopkins University\\
    $^2$Department of Computer Science, Johns Hopkins University\\
    $^3$Department of Biomedical Engineering, Johns Hopkins University\\
    \rule{0.05\linewidth}{.1pt}\\
    \begin{minipage}{0.6\linewidth}
        \centering
        \begin{align*}
        \text{code:}    &&~\text{\scriptsize\url{https://github.com/Sulam-Group/semantic_uq}}
        \end{align*}
    \end{minipage}
}

\begin{document}
\maketitle
\begin{abstract}
Uncertainty quantification is necessary for developers, physicians, and regulatory agencies to build trust in machine learning predictors and improve patient care. Beyond measuring uncertainty, it is crucial to express it in clinically meaningful terms that provide actionable insights. This work introduces a conformal risk control (CRC) procedure for organ-dependent uncertainty estimation, ensuring high-probability coverage of the ground-truth image. We first present a high-dimensional CRC procedure that leverages recent ideas of length minimization. We make this procedure semantically adaptive to each patient's anatomy and positioning of organs. Our method, $\semCRC$, provides tighter uncertainty intervals with valid coverage on real-world computed tomography (CT) data while communicating uncertainty with clinically relevant features.
\end{abstract}

\section{Introduction}
Deep learning predictors are becoming ubiquitous in solving inverse problems in medical imaging, with remarkable performance across diverse modalities and organ systems. Point predictors, however, are limited in their ability to quantify uncertainty, as is often necessary for developers, physicians, and regulatory agencies to verify the safety and reliability of these models in real-world clinical settings. For example, it has been shown that diffusion models can hallucinate the details of a patient's anatomy \cite{tivnan2024hallucination,webber2024diffusion}, and robust notions of predictive uncertainty could ameliorate these issues. At the same time, several studies have highlighted the benefits of including uncertainty estimates in computer-aided decision making processes \cite{mccrindle2021radiology,faghani2023quantifying,maruccio2024clinical,salvi2025explainability}. This motivates communicating uncertainty in a clinically informed or clinically relevant manner.

Conformal risk control (CRC) \cite{angelopoulos2024conformal} addresses the challenges of measuring the uncertainty of black-box systems without assuming a predictive distribution, having found numerous applications in medicine \cite{hulsman2024conformal,angelopoulos2024conformaltriage,kutiel2023conformal,teneggi2023trust,angelopoulos2022image}. In imaging, CRC constructs pixel-wise intervals by starting from heuristic notions of uncertainty (e.g., quantile regression \cite{koenker1978regression}, MC-Dropout \cite{gal2016dropout}, or variance of the samples from a diffusion model \cite{teneggi2023trust}), and then conformalizing the resulting sets to achieve risk control. How to minimize interval length in high-dimensional settings is the subject of ongoing research \cite{kiyani2024length,bars2025volume,belhasin2023principal,nehme2023uncertainty}.

In this work, we observe that patients' anatomies vary in size, shape, and positioning of organs, and these variations may unintentionally inflate interval length. We propose to construct \emph{organ-dependent} uncertainty intervals that encompass semantic structures beyond pixels. We achieve this by extending the CRC-equivalent of the $\kRCPS$ procedure \cite{teneggi2023trust}, minimizing the mean interval length via convex optimization. Not only does our method, $\semCRC$, provide tighter intervals, but it can also guarantee the same level of risk control for each organ rather than cumulatively over a scan. We evaluate our method on quantile regression for CT denoising and a simple FBP-UNet reconstruction pipeline using two real-world datasets: TotalSegmentator \cite{wasserthal2023totalsegmentator} and FLARE23 \cite{ma2022fast}. Our contributions apply broadly to any imaging inverse problem and any predictor equipped with a heuristic notion of uncertainty.

\section{Background}
Recall that in inverse problems, we aim to retrieve an underlying signal $X \in \X$ from measurements $Y \in \Y$, where $Y = \A(X)$ and the operator $\A: \X \to \Y$ cannot be directly inverted (e.g., due to being ill-posed or affected by noise). Herein, we let $\X$ be the space of $d$-dimensional images, i.e. $\X \subseteq [0,1]^d$.

\paragraph{\textbf{Quantile regression.}} A common approach to solving inverse problems is to train a \emph{point predictor} $f: \Y \to \X$ that minimizes a loss function $L(f(y), x)$ over a dataset $\{(X\at{i}, Y\at{i})\}_{i=1}^n$ of ground-truth signals with their measurements. For example, if $L$ is the squared error then $f(Y) \approx \E[X \mid Y]$. Differently, quantile regression trains a \emph{set predictor} $g: \Y \to 2^{\X}$ such that $\forall j \in [d]$, $g(y)_j = [\q_{\alpha}(y)_j, \q_{1-\alpha}(y)_j]$ where $\q_t(Y)_j$ is the estimate of the $t$-level quantile of $\P[X_j \mid Y]$, which can be learned by minimizing the pinball loss \cite{koenker1978regression}. Thus, quantile regression provides an estimate of uncertainty with intervals length.

\paragraph{\textbf{Conformal risk control (CRC).}} The goal of conformal risk control \cite{angelopoulos2024conformal} is to post-process a fixed set predictor $g$ to bound the expectation of its error. More formally, denote $\{g_{\lambda}\}_{\lambda \in \R_{\geq 0}}$ the family of nested predictors with
\begin{equation}
    g_{\lambda}(y)_j = [\q_{\alpha}(y)_j - \lambda, \q_{1-\alpha}(y)_j + \lambda],
\end{equation}
and let $\l(g_{\lambda}(y),x)$ be any bounded, non-increasing function of $\lambda$. Following prior work \cite{angelopoulos2022image,teneggi2023trust}, we will consider the proportion of ground-truth pixels that fall outside of their intervals, i.e.
\begin{equation}
    \label{eq:l01}
    \l^{01}(g_{\lambda}(y),x) = \frac{1}{d} \sum_{j \in [d]} \1\{x_j \notin g_{\lambda}(y)_j\},
\end{equation}
which is monotonically non-increasing in $\lambda$ and bounded by 1. Then, for any tolerance $\epsilon > 0$, one can find the parameter $\hlambda$ that controls the loss in \eqref{eq:l01}. In particular, given a calibration set $S_{\cal} = \{(X\at{i}, Y\at{i})\}_{i=1}^{n_{\cal}}$, and a test point $(X,Y)$ of exchangeable observations independent of $g$, the choice of
\begin{equation}
    \hlambda = \inf \left\{\lambda \in \R_{\geq 0}: \frac{n_{\cal}}{n_{\cal} + 1} \hl^{01}_{\cal}(\lambda) + \frac{1}{n_{\cal} + 1} \leq \epsilon\right\}
\end{equation}
where $\hl^{01}_{\cal}(\lambda) = 1/n_{\cal} \sum_{(x,y) \in S_{\cal}} \l^{01}(g_{\lambda}(y),x)$ guarantees that
\begin{equation}
    \label{eq:crc}
    \E[\l^{01}(g_{\hlambda}(Y),X)] \leq \epsilon,
\end{equation}
where the expectation is taken over $S_{\cal}$ and $(X,Y)$.

\paragraph{\textbf{High-dimensional risk control.}} As noted by \cite{teneggi2023trust}, using the same scalar $\lambda$ for all pixels inflates the mean interval length of the conformalized sets. To overcome this limitation, they propose to assign each pixel to one of $K$ groups with some shared statistics. More precisely, they consider a partition matrix $M \in \{0,1\}^{d \times K}$, and use a vector-valued parameter $\vlambda_K = [\lambda_1, \dots, \lambda_K] \in \R^K_{\geq 0}$ such that $\vlambda = M\vlambda_K \in \R^d_{\geq 0}$ and
\begin{equation}
    \label{eq:pixel_uq}
    g_{\vlambda}(y)_j = [\q_{\alpha}(y)_j - \lambda_j, \q_{1-\alpha}(y)_j + \lambda_j].
\end{equation}
Then, for a fixed anchor point $\tvlambda_K \in \R^K_{\geq 0}$, choosing
\begin{equation}
    \label{eq:krcps}
    \hvlambda = \inf\left\{\vlambda \in M\tvlambda_K + \omega \1_d,~\omega \in \R: \frac{n_{\cal}}{n_{\cal}+1} \hl^{01}_{\cal}(\vlambda) + \frac{1}{n_{\cal} + 1} \leq \epsilon\right\}
\end{equation}
controls risk as in \eqref{eq:crc}. Note that \cite{teneggi2023trust} introduced their method for risk controlling prediction sets (RCPSs) \cite{bates2021distribution}, but it applies to CRC as well. The anchor $\tvlambda_K \in \R^K_{\geq 0}$ is arbitrary, but it should be chosen to minimize the mean interval length. The proposed method, $\kCRC$, introduces $\l^{\gamma}$ for $\gamma \in (0,1)$: a convex upper-bound to $\l^{01}$. Then, it solves the following optimization problem
\begin{equation}\tag{P$K$}
    \label{eq:pk}
    \tvlambda_{K} = \argmin_{\vlambda_K \in \R^K_{\geq 0}}~\sum_{k \in [K]} n_k\lambda_k~\quad\st\quad~\hl^{\gamma}_{\opt}(M\vlambda_K)\leq \epsilon,
\end{equation}
where $n_k$ is the number of pixels in group $k$. We stress that in this procedure, the calibration set $S_{\cal}$ needs to be split in $S_{\opt}$ and $\widetilde{S}_{\cal}$, such that the former is used to solve \eqref{eq:pk} and the latter to find $\hvlambda$ as in \eqref{eq:krcps}.

With this background, we now present the main contributions of our work.

\section{Semantic Uncertainty Quantification}
Observe that the partition matrix $M$ that assigns each of the $d$ pixels to one of $K$ groups does not depend on the measurement, $y$. This choice is effective when the semantic content of each pixel is similar across observations (e.g., face images can be aligned and centered). However, CT data is heterogeneous, and a fixed partition matrix may unnecessarily increase the mean interval length.

In this work, we leverage foundational segmentation models \cite{qu2023abdomenatlas,li2024well} to construct organ-dependent uncertainty intervals. Our method, $\semCRC$, extends $\kCRC$ to instance-dependent memberships $s(y) \in [K]^d$. This decouples optimizing the mean interval length from the pixel domain, and it reflects the uncertainty of the model in terms of semantic---and clinically meaningful---structures. Formally, let $s: \Y \to [K]^d$ be a fixed segmentation model such that, for a vector $\vlambda_{\sem} \in \R^K_{\geq 0}$, the family of nested set predictors $\{g_{\vlambda_{\sem}}\}$ is given by
\begin{equation}
    \label{eq:sem_uq}
    g_{\vlambda_{\sem}}(y)_j = [\q_{\alpha}(y)_j - \lambda_{s(y)_j}, \q_{1-\alpha}(y)_j + \lambda_{s(y)_j}].
\end{equation}
Note that, differently from $g_{\vlambda}(y)$ in \eqref{eq:pixel_uq}, the same pixel $j$ may receive different assignments in different scans depending on the measurement $y$. Our work does not study the performance of $s$, and calibration of segmentation models is subject of complementary research \cite{mossina2024conformal,wundram2024conformal,davenport2024conformal,brunekreef2024kandinsky}. We will proceed analogously to the above, i.e. finding an anchor $\tvlambda_{\sem}$ that minimizes the mean interval length $\I_{\vlambda_{\sem}}(y)$, and then backtracking along the line $\tvlambda_{\sem} + \omega \1_K$ to control risk. 

\begin{figure}[t]
    \centering
    \subcaptionbox{TotalSegmentator.}{\includegraphics[width=\linewidth]{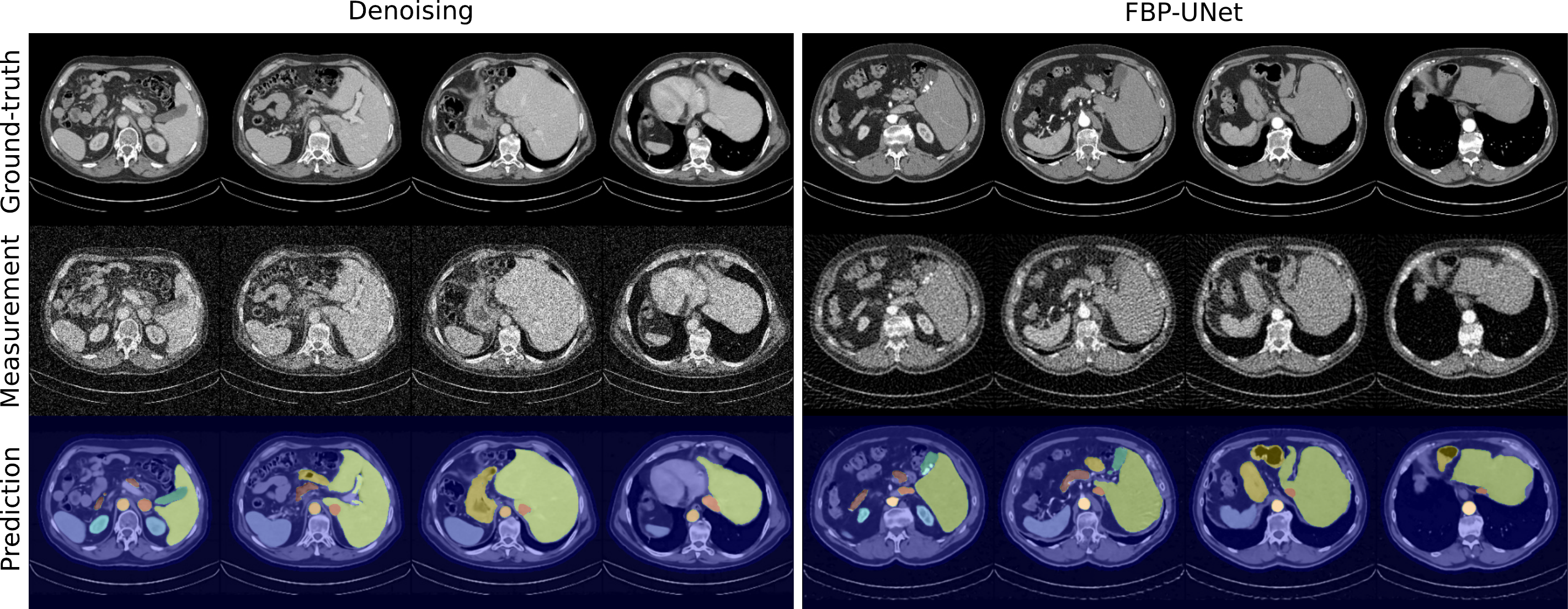}}
    \subcaptionbox{FLARE23.}{\includegraphics[width=\linewidth]{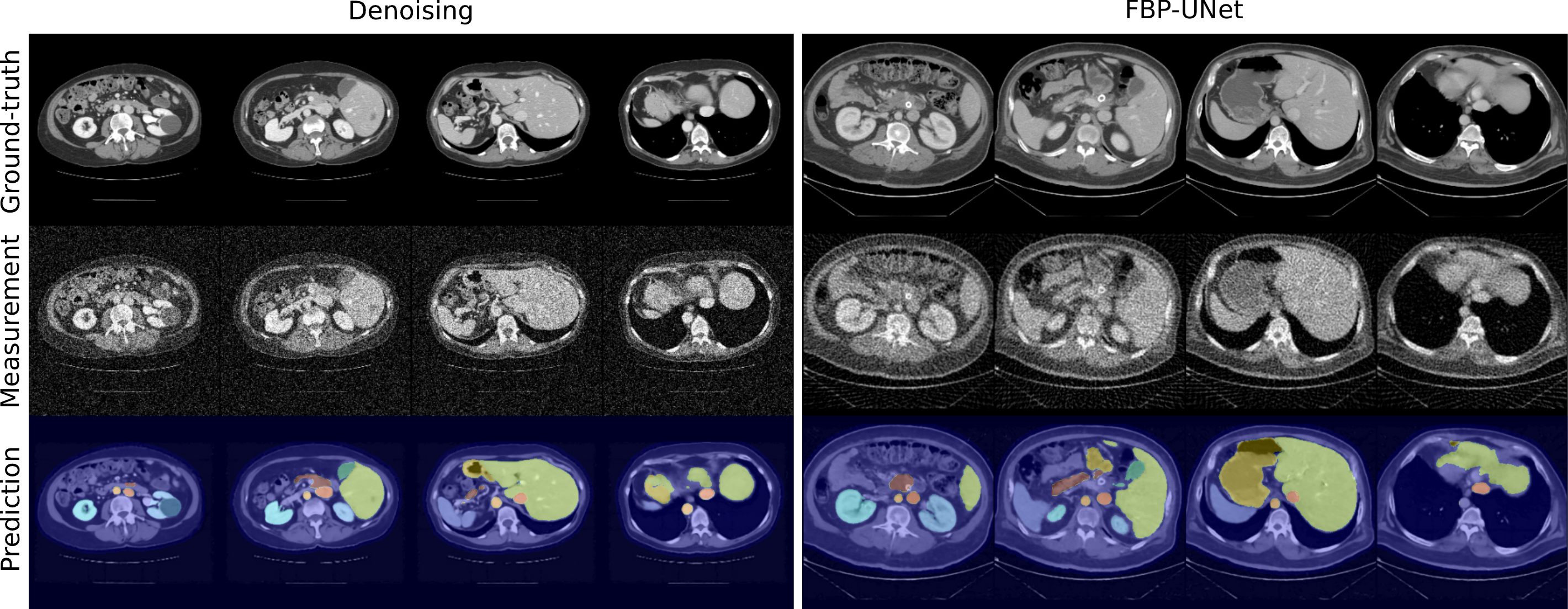}}
    \caption{\label{fig:prediction_results}Example calibration data: ground-truth, measurement, and segmented predictions for both tasks and datasets.}
\end{figure}

Start by noting that
\begin{equation}
    \I_{\vlambda_{\sem}}(y) = \frac{1}{d} \sum_{j \in [d]} (\q_{1-\alpha}(y)_j - \q_{\alpha}(y)_j) + \frac{1}{d} \sum_{k \in [K]} \lvert \S_k(y) \rvert\lambda_k
\end{equation}
where $\S_k(y) = \{j \in [d]: s(y)_j = k\}$ is the set of voxels that belong to organ $k$ for observation $y$. We can see that the mean interval length is still a function of the sum of the $\lambda_k$'s, but the multiplicative factors now depend on $y$ as well. So, it becomes necessary to minimize the mean interval length in expectation over $Y$. We extend the original optimization problem \eqref{eq:pk} to its semantic version
\begin{equation}\tag{P$sem$}
    \label{eq:psem}
    \tvlambda_{\sem} = \argmin_{\vlambda_{\sem} \in \R^K_{\geq 0}}~\sum_{k \in [K]} \E_Y[\lvert \S_k(Y) \rvert]\lambda_k~\quad\st\quad~\hl^{\gamma}_{\opt}(\vlambda_{\sem}) \leq \epsilon,
\end{equation}
where, in practice, we estimate the objective over $S_{\opt}$. To conclude, we choose
\begin{equation}
    \label{eq:semrcps}
    \hvlambda_{\sem} = \inf \left\{\vlambda_{\sem} \in \tvlambda_{\sem} + \omega \1_K: \frac{n_{\cal}}{n_{\cal} + 1} \hl^{01}_{\cal}(\vlambda_{\sem}) + \frac{1}{n_{\cal} + 1} \leq \epsilon\right\},
\end{equation}
and we state the validity of $\semCRC$ in the following proposition.\footnote{We present results for CRC, but our method generalizes to RCPSs as well.}
\begin{proposition}
    \label{prop:semrcps}
    For a risk tolerance $\epsilon > 0$, segmentation model $s: \Y \to [K]^d$, anchor point $\tvlambda_{\sem} \in \R^K_{\geq 0}$, and exchangeable calibration and test points $S_{\cal} = \{(X\at{i},Y\at{i})\}_{i=1}^{n_{\cal}}$, $(X,Y)$, the choice of $\hvlambda_{\sem}$ as in \eqref{eq:semrcps} provides risk control, i.e.
    \begin{equation}
        \E[\l^{01}(g_{\hvlambda_{\sem}}(Y),X)] \leq \epsilon.
    \end{equation}
\end{proposition}
\begin{proof}
    Let $\vlambda_{\sem}(\omega) = \tvlambda_{\sem} + \omega\1_K$, $\omega \in \R$, and note that $\l^{01}(g_{\vlambda_{\sem}(\omega)}(y),x)$ is bounded by 1 and monotonically non-increasing in $\omega$. Since $s$ is fixed, the random functions $L_i(\omega) = \l^{01}(g_{\vlambda_{\sem}(\omega)}(Y\at{i}),X\at{i})$ and $L(\omega) = \l^{01}(g_{\vlambda_{\sem}(\omega)}(Y),X)$ are exchangeable. The result then follows by applying \cite[Theorem 1]{angelopoulos2024conformal} to $\omega$.
\end{proof}
We remark that $\semCRC$ also relies on splitting the calibration set $S_{\cal}$ into $S_{\opt}$ to solve \eqref{eq:psem}, and $\widetilde{S}_{\cal}$ to find $\hvlambda_{\sem}$ as in \eqref{eq:semrcps}. Furthermore, and naturally, the method requires performing inference with the same segmentation model used for calibration. We regard semantic calibration with respects to ground-truth segmentations as an extension of this work.

\paragraph{\textbf{Controlling risk for each organ.}} Clinical tasks may require different organs to have the same level of reconstruction accuracy, but $\hvlambda_{\sem}$ may overcover easy-to-reconstruct ones while undercovering others. Thus, we specialize $\semCRC$ to control risk with the same tolerance $\epsilon$ for each segmented structures, and we call this variation $\bsemCRC$. Denote
\begin{equation}
    \l^{01}_{k}(g_{\vlambda_{\sem}}(y), x) = \frac{1}{\lvert\S_k(y)\rvert} \sum_{j \in \S_k(y)} \1\{x_j \notin g_{\vlambda_{\sem}}(y)_j\}
\end{equation}
the proportion of pixels in organ $k$ (e.g., liver) that fall outside of their intervals, and let $\ve_k$ be the $k^{\text{th}}$ standard basis vector. The choice of $\hvlambda_{\bsem} \in \R^K_{\geq 0}$ with 
\begin{equation}
    \hlambda_{\bsem,j} = \inf\left\{\lambda \in \R_{\geq 0}: \frac{n_{\cal}}{n_{\cal} + 1} \hl^{01}_{k,\cal}(\tvlambda_{\sem} + \lambda\ve_k) + \frac{1}{n_{\cal} + 1} \leq \epsilon\right\}
\end{equation}
provides risk control for each organ, that is $\E[\l^{01}_k(g_{\hvlambda_{\bsem}}(Y), X)] \leq \epsilon$, $k = 1, \dots, K$. This follows by applying \cref{prop:semrcps} to each dimension of $\hvlambda_{\bsem}$. We briefly remark this is different from multiple risk control with one scalar $\lambda$ as in \cite{angelopoulos2024conformal}, and that the equivalent for RCPS requires multiple hypothesis testing correction for uniform coverage.

\begin{table}[t]
\centering
\caption{\label{table:results}Summary of calibration results as mean and standard deviation over 20 independent runs of each calibration procedure with risk tolerance $\epsilon = 0.10$.}
\begin{small}
\begin{tabular}{llcccc}
\toprule
                                    &               & \multicolumn{2}{c}{TotalSegmentator}                      & \multicolumn{2}{c}{FLARE23} \\ 
                                                    \cmidrule(r){3-4}                                           \cmidrule(r){5-6}
Task                                & Procedure     & Risk              & Length {\smallsize($\times 10^{-2}$)} & Risk              & Length {\smallsize ($\times 10^{-2})$} \\ \hline
\multirow{4}{*}{Denoising}     & $\CRC$        & $0.095 \pm 0.006$ & $11.60 \pm 0.21$                      & $0.096 \pm 0.004$ & $9.16 \pm 0.09$\\ 
                                    & $\kCRC$       & $0.097 \pm 0.006$ & $9.37 \pm 0.20$                       & $0.096 \pm 0.006$ & $6.81 \pm 0.21$\\ 
                                    & $\semCRC$     & $0.098 \pm 0.006$ & $\bm{8.72 \pm 0.18}$                  & $0.095 \pm 0.006$ & $\bm{6.36 \pm 0.11}$\\
                                    & $\bsemCRC$    & $0.055 \pm 0.004$ & $11.84 \pm 0.20$                      & $0.056 \pm 0.003$ & $8.06 \pm 0.16$\\
\midrule
\multirow{4}{*}{FBP-UNet}   & $\CRC$        & $0.098 \pm 0.007$ & $10.43 \pm 0.23$                      & $0.095 \pm 0.006$ & $6.19 \pm 0.09$ \\  
                                    & $\kCRC$       & $0.098 \pm 0.009$ & $9.32 \pm 0.13$                       & $0.095 \pm 0.003$ & $6.20 \pm 0.14$ \\  
                                    & $\semCRC$     & $0.097 \pm 0.007$ & $\bm{8.95 \pm 0.19}$                  & $0.095 \pm 0.006$ & $\bm{6.18 \pm 0.13}$ \\  
                                    & $\bsemCRC$    & $0.059 \pm 0.005$ & $12.43 \pm 0.20$                      & $0.057 \pm 0.003$ & $7.72 \pm 0.17$\\
\bottomrule
\end{tabular}
\end{small}
\end{table}

\section{Experiments}
We compare CRC, $\kCRC$, and $\semCRC$ for denoising and for a basic FBP-UNet reconstruction task on TotalSegmentator \cite{wasserthal2023totalsegmentator} ($1,429$ scans) and on the first 1,000 scans from the training split of the FLARE23 \cite{ma2022fast} challenge. We resample the FLARE23 dataset at \qtyproduct{1.5 x 1.5 x 3.0}{\mm} resolution, and we window all scans between $-175$\,HU and $250$\,HU. For denoising, we add independent Gaussian noise with $\sigma = 0.2$; for reconstruction, we use the ODL library \cite{odl} with ASTRA \cite{van2016fast,van2015astra} to simulate a helical cone beam geometry. We set the pitch adaptively to cover the entire volume in 8 turns, and acquire data over 1,000 angles with a detector of size \qtyproduct{512 x 128}{pixels}. We model low-dose measurement as linear Poisson noise with $I_0 = 1,000$. We chose these settings to highlight our method's performance on a challenging task. For each task, we use MONAI \cite{cardoso2022monai} to train a 3D UNet \cite{ronneberger2015u} ($\approx$ 5\,M parameters, ROI of $96^3$ voxels) with quantile regression ($\alpha = 0.1$, i.e. the $10^{\text{th}}$ and $90^{\text{th}}$ quantiles) on the AbdomenAtlas-8K dataset \cite{qu2023abdomenatlas} ($5,195$ scans).

\begin{figure}[t]
    \centering
    \subcaptionbox{TotalSegmentator.}{\includegraphics[width=0.49\linewidth]{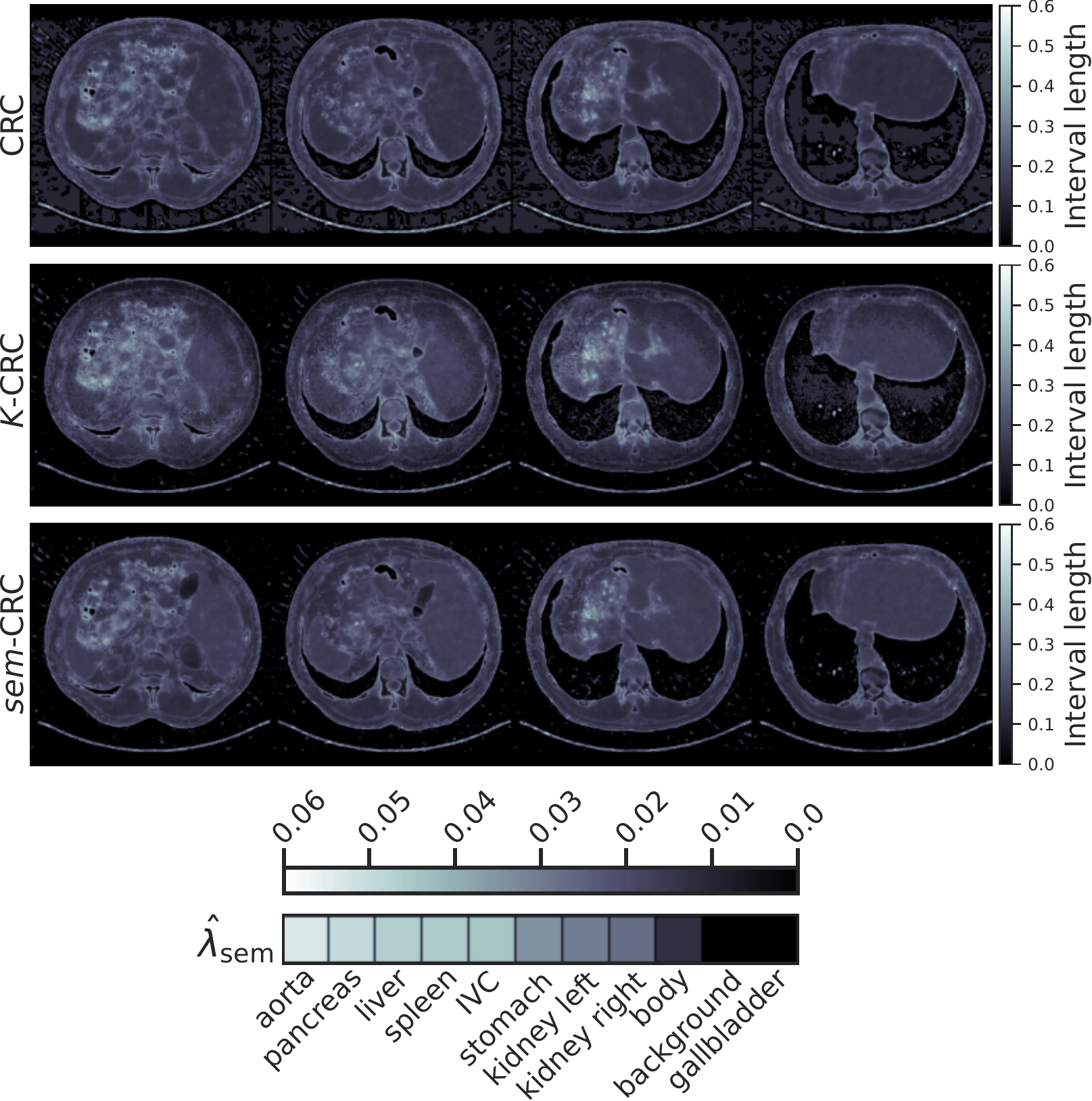}}
    \subcaptionbox{FLARE23.}{\includegraphics[width=0.49\linewidth]{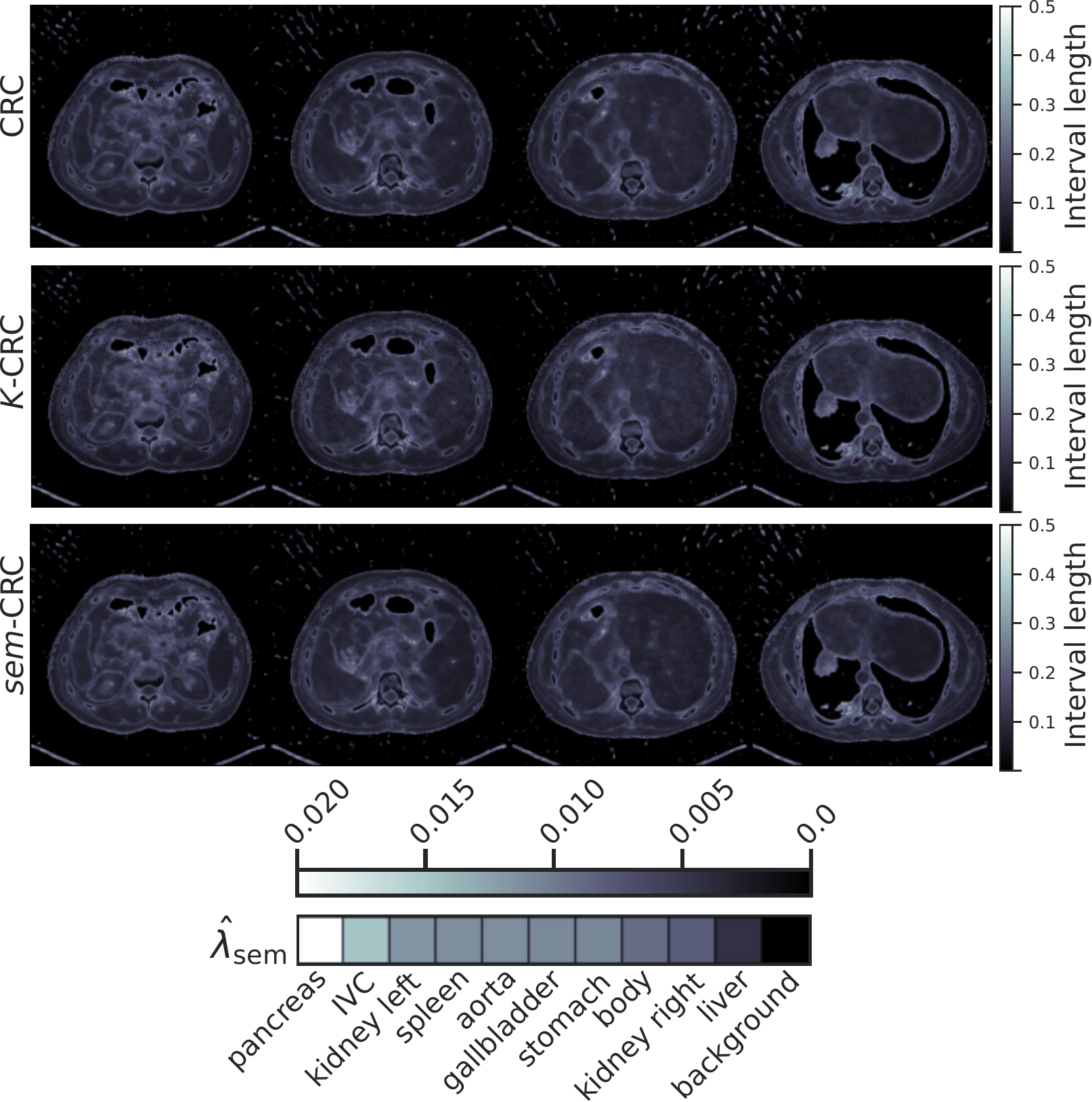}}
    \caption{\label{fig:uncertainty_maps}Example conformalized uncertainty maps on one volume per dataset with each calibration method for the FBP-UNet pipeline. The bottom row shows $\hvlambda_{\sem}$, the semantic uncertainty parameter learned by our method, $\semCRC$.}
\end{figure}

We segment 9 structures: spleen, kidneys, gallbladder, liver, stomach, aorta, inferior vena cava (IVC), and pancreas using SuPrem \cite{li2024well}, a state-of-the-art general-purpose segmentation model for medical imaging. All remaining voxels that are not background are labeled generically as ``body''. To solve \eqref{eq:psem} over a distribution of volumes that represents all organs, we select 4 equidistant slices from the window of $48$ that maximizes the segmentation volume. Finally, we center-crop or pad slices to \qtyproduct{256 x 256}{voxels} for calibration.

Since $\semCRC$ relies on a fixed segmentation model, we evaluate predictions in terms of mean structure-wise F1 score between the segmented outputs and the ground-truth annotations over 200 random volumes. For the TotalSegmentator dataset, we obtain $0.85 \pm 0.07$ and $0.83 \pm 0.08$ for denoising and FBP-UNet, respectively; and, equivalently, $0.88 \pm 0.06$ and $0.87 \pm 0.07$ for the FLARE23 dataset. Although we see a slight drop in performance compared to the metrics reported in \cite{li2024well}, these results confirm predictions are of reasonable quality for segmentation, and we include some examples in \cref{fig:prediction_results}.

\begin{figure}[t]
    \centering
    \subcaptionbox{TotalSegmentator.}{\includegraphics[width=0.49\linewidth]{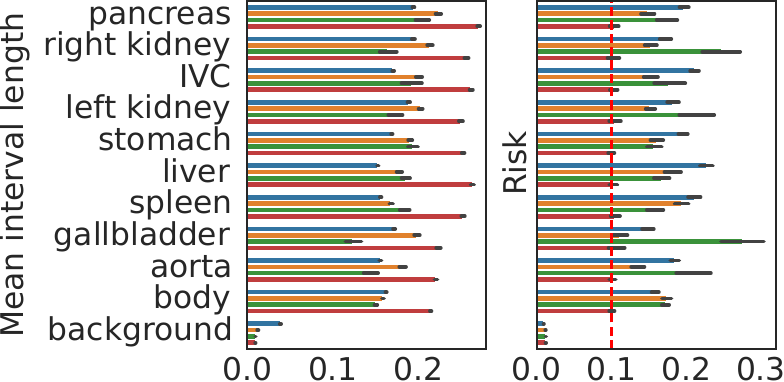}}
    \hfill
    \subcaptionbox{FLARE23.}{\includegraphics[width=0.49\linewidth]{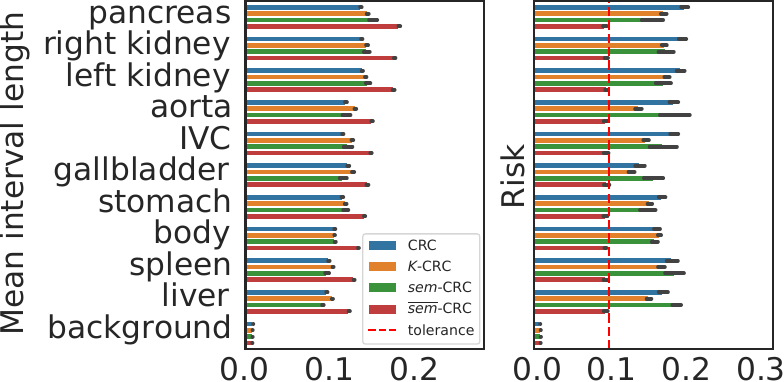}}
    \caption{\label{fig:organ_loss}Mean interval length and risk stratified by organ for the FBP-UNet task across all calibration procedures and datasets. $\bsemCRC$ is the only procedure that guarantees risk control for each organ.}
\end{figure}

We set the error tolerance to $\epsilon = 0.10$, allowing at most 10\% of ground-truth voxels to fall outside their prediction intervals. Each calibration procedure is run 20 times on independent subsets of $n_{\cal} = 512$ scans, with risk estimated on $n_{\text{test}} = 128$ scans. We allocate $n_{\opt} = 32$ calibration samples to solve \eqref{eq:pk} and \eqref{eq:psem}, ensuring a fair comparison across methods. For $\kCRC$, we follow \cite{teneggi2023trust} and construct the assignment matrix $M$ by grouping voxels into $K = 4$ quantiles of the loss on the optimization set. Finally, to solve \eqref{eq:pk} efficiently, we subsample $d_{\opt} = 50$ voxels (much smaller than $256^2$) stratified by membership; and for \eqref{eq:psem}, we ensure the smallest organ has a support of at least $d_{\min} = 2$ voxels by subsampling $d_{\opt} = d_{\min} / \min_{k} \E[\lvert \S_k \rvert]$ dimensions ($d_{\opt} \approx 3,000$). Solving subsampled problems reduces complexity to the order of seconds. 

\cref{table:results} summarizes risk and mean interval length across all datasets and tasks. All procedures are \emph{valid}, i.e. they control risk at level $\epsilon$. Our method, $\semCRC$, consistently provides the shortest uncertainty intervals. On the other hand, and as expected, controlling risk for each organ with $\bsemCRC$ increases the mean interval length. \cref{fig:uncertainty_maps} compares the conformalized uncertainty maps obtained with each method on the same volume, and it includes the vector $\hvlambda_{\sem}$ learned by $\semCRC$. The uncertainty maps generated by $\semCRC$ are sharper and contain fewer artifacts thanks to using instance-level information. Furthermore, $\hvlambda_{\sem}$ directly informs on which organs have higher levels of uncertainty, depicting how the same model may display different uncertainty patterns across different populations. These findings are fundamental to the responsible use of general-purpose machine learning models across centers serving diverse demographics. Finally, \cref{fig:organ_loss} highlights the difference between controlling risk for each organ or cumulatively over a volume: all methods but $\bsemCRC$ achieve risk control by overcovering background and undercovering organs. Our methodology gives users the flexibility to specify which organs they desire to control risk for depending on the clinical task at hand.

\section{Conclusions}
Modern deep learning models are widely used for image reconstruction, including computed tomography. However, they often provide only point-wise estimates, lacking statistically valid uncertainty measures. This work proposes a conformal prediction approach that generates uncertainty intervals with controlled risk at any user-specified level. By integrating high-dimensional calibration and state-of-the-art segmentation models, our method, $\semCRC$, produces organ-dependent uncertainty sets that are adaptive to each patient. Moreover, it can control risk \emph{for each organ}. Not only does $\semCRC$ provide the tightest uncertainty set, but also it communicates findings with clinically meaningful anatomical structures.

\subsubsection*{Acknowledgments}
This research was supported by NSF CAREER Award CCF 2239787.

\bibliographystyle{plainnat}
\bibliography{bibliography}

\begin{thebibliography}{32}
\providecommand{\natexlab}[1]{#1}
\providecommand{\url}[1]{\texttt{#1}}
\expandafter\ifx\csname urlstyle\endcsname\relax
  \providecommand{\doi}[1]{doi: #1}\else
  \providecommand{\doi}{doi: \begingroup \urlstyle{rm}\Url}\fi

\bibitem[Adler et~al.(2017)Adler, Kohr, and Öktem]{odl}
Jonas Adler, Holger Kohr, and Ozan Öktem.
\newblock Operator discretization library (odl), 2017.
\newblock URL \url{https://github.com/odlgroup/odl}.

\bibitem[Angelopoulos et~al.(2022)Angelopoulos, Kohli, Bates, Jordan, Malik, Alshaabi, Upadhyayula, and Romano]{angelopoulos2022image}
Anastasios~N Angelopoulos, Amit~Pal Kohli, Stephen Bates, Michael Jordan, Jitendra Malik, Thayer Alshaabi, Srigokul Upadhyayula, and Yaniv Romano.
\newblock Image-to-image regression with distribution-free uncertainty quantification and applications in imaging.
\newblock In \emph{International Conference on Machine Learning}, pages 717--730. PMLR, 2022.

\bibitem[Angelopoulos et~al.(2024{\natexlab{a}})Angelopoulos, Pomerantz, Do, Bates, Bridge, Elton, Lev, Gonz{\'a}lez, Jordan, and Malik]{angelopoulos2024conformaltriage}
Anastasios~N Angelopoulos, Stuart Pomerantz, Synho Do, Stephen Bates, Christopher~P Bridge, Daniel~C Elton, Michael~H Lev, R~Gilberto Gonz{\'a}lez, Michael~I Jordan, and Jitendra Malik.
\newblock Conformal triage for medical imaging ai deployment.
\newblock \emph{medRxiv}, pages 2024--02, 2024{\natexlab{a}}.

\bibitem[Angelopoulos et~al.(2024{\natexlab{b}})Angelopoulos, Bates, Fisch, Lei, and Schuster]{angelopoulos2024conformal}
Anastasios~Nikolas Angelopoulos, Stephen Bates, Adam Fisch, Lihua Lei, and Tal Schuster.
\newblock Conformal risk control.
\newblock In \emph{The Twelfth International Conference on Learning Representations}, 2024{\natexlab{b}}.

\bibitem[Bars and Humbert(2025)]{bars2025volume}
Batiste~Le Bars and Pierre Humbert.
\newblock On volume minimization in conformal regression.
\newblock \emph{arXiv preprint arXiv:2502.09985}, 2025.

\bibitem[Bates et~al.(2021)Bates, Angelopoulos, Lei, Malik, and Jordan]{bates2021distribution}
Stephen Bates, Anastasios Angelopoulos, Lihua Lei, Jitendra Malik, and Michael Jordan.
\newblock Distribution-free, risk-controlling prediction sets.
\newblock \emph{Journal of the ACM (JACM)}, 68\penalty0 (6):\penalty0 1--34, 2021.

\bibitem[Belhasin et~al.(2023)Belhasin, Romano, Freedman, Rivlin, and Elad]{belhasin2023principal}
Omer Belhasin, Yaniv Romano, Daniel Freedman, Ehud Rivlin, and Michael Elad.
\newblock Principal uncertainty quantification with spatial correlation for image restoration problems.
\newblock \emph{IEEE Transactions on Pattern Analysis and Machine Intelligence}, 46\penalty0 (5):\penalty0 3321--3333, 2023.

\bibitem[Brunekreef et~al.(2024)Brunekreef, Marcus, Sheombarsing, Sonke, and Teuwen]{brunekreef2024kandinsky}
Joren Brunekreef, Eric Marcus, Ray Sheombarsing, Jan-Jakob Sonke, and Jonas Teuwen.
\newblock Kandinsky conformal prediction: efficient calibration of image segmentation algorithms.
\newblock In \emph{Proceedings of the IEEE/CVF Conference on Computer Vision and Pattern Recognition}, pages 4135--4143, 2024.

\bibitem[Cardoso et~al.(2022)Cardoso, Li, Brown, Ma, Kerfoot, Wang, Murrey, Myronenko, Zhao, Yang, et~al.]{cardoso2022monai}
M~Jorge Cardoso, Wenqi Li, Richard Brown, Nic Ma, Eric Kerfoot, Yiheng Wang, Benjamin Murrey, Andriy Myronenko, Can Zhao, Dong Yang, et~al.
\newblock Monai: An open-source framework for deep learning in healthcare.
\newblock \emph{arXiv preprint arXiv:2211.02701}, 2022.

\bibitem[Davenport(2024)]{davenport2024conformal}
Samuel Davenport.
\newblock Conformal confidence sets for biomedical image segmentation.
\newblock \emph{arXiv preprint arXiv:2410.03406}, 2024.

\bibitem[Faghani et~al.(2023)Faghani, Moassefi, Rouzrokh, Khosravi, Baffour, Ringler, and Erickson]{faghani2023quantifying}
Shahriar Faghani, Mana Moassefi, Pouria Rouzrokh, Bardia Khosravi, Francis~I Baffour, Michael~D Ringler, and Bradley~J Erickson.
\newblock Quantifying uncertainty in deep learning of radiologic images.
\newblock \emph{Radiology}, 308\penalty0 (2):\penalty0 e222217, 2023.

\bibitem[Gal and Ghahramani(2016)]{gal2016dropout}
Yarin Gal and Zoubin Ghahramani.
\newblock Dropout as a bayesian approximation: Representing model uncertainty in deep learning.
\newblock In \emph{international conference on machine learning}, pages 1050--1059. PMLR, 2016.

\bibitem[Hulsman et~al.(2024)Hulsman, Comte, Bertolini, Wiesenthal, Gallardo, and Ceresa]{hulsman2024conformal}
Roel Hulsman, Valentin Comte, Lorenzo Bertolini, Tobias Wiesenthal, Antonio~Puertas Gallardo, and Mario Ceresa.
\newblock Conformal risk control for pulmonary nodule detection.
\newblock \emph{arXiv preprint arXiv:2412.20167}, 2024.

\bibitem[Kiyani et~al.(2024)Kiyani, Pappas, and Hassani]{kiyani2024length}
Shayan Kiyani, George Pappas, and Hamed Hassani.
\newblock Length optimization in conformal prediction.
\newblock \emph{arXiv preprint arXiv:2406.18814}, 2024.

\bibitem[Koenker and Bassett~Jr(1978)]{koenker1978regression}
Roger Koenker and Gilbert Bassett~Jr.
\newblock Regression quantiles.
\newblock \emph{Econometrica: journal of the Econometric Society}, pages 33--50, 1978.

\bibitem[Kutiel et~al.(2023)Kutiel, Cohen, Elad, Freedman, and Rivlin]{kutiel2023conformal}
Gilad Kutiel, Regev Cohen, Michael Elad, Daniel Freedman, and Ehud Rivlin.
\newblock Conformal prediction masks: Visualizing uncertainty in medical imaging.
\newblock In \emph{International Workshop on Trustworthy Machine Learning for Healthcare}, pages 163--176. Springer, 2023.

\bibitem[Li et~al.(2024)Li, Yuille, and Zhou]{li2024well}
Wenxuan Li, Alan Yuille, and Zongwei Zhou.
\newblock How well do supervised models transfer to 3d image segmentation.
\newblock In \emph{The Twelfth International Conference on Learning Representations}, volume~1, 2024.

\bibitem[Ma et~al.(2022)Ma, Zhang, Gu, An, Wang, Ge, Wang, Zhang, Wang, Xu, et~al.]{ma2022fast}
Jun Ma, Yao Zhang, Song Gu, Xingle An, Zhihe Wang, Cheng Ge, Congcong Wang, Fan Zhang, Yu~Wang, Yinan Xu, et~al.
\newblock Fast and low-gpu-memory abdomen ct organ segmentation: the flare challenge.
\newblock \emph{Medical Image Analysis}, 82:\penalty0 102616, 2022.

\bibitem[Maruccio et~al.(2024)Maruccio, Eppinga, Laves, Navarro, Salvi, Molinari, and Papaconstadopoulos]{maruccio2024clinical}
Federica~Carmen Maruccio, Wietse Eppinga, Max-Heinrich Laves, Roger~Fonolla Navarro, Massimo Salvi, Filippo Molinari, and Pavlos Papaconstadopoulos.
\newblock Clinical assessment of deep learning-based uncertainty maps in lung cancer segmentation.
\newblock \emph{Physics in Medicine \& Biology}, 69\penalty0 (3):\penalty0 035007, 2024.

\bibitem[McCrindle et~al.(2021)McCrindle, Zukotynski, Doyle, and Noseworthy]{mccrindle2021radiology}
Brian McCrindle, Katherine Zukotynski, Thomas~E Doyle, and Michael~D Noseworthy.
\newblock A radiology-focused review of predictive uncertainty for ai interpretability in computer-assisted segmentation.
\newblock \emph{Radiology: Artificial Intelligence}, 3\penalty0 (6):\penalty0 e210031, 2021.

\bibitem[Mossina et~al.(2024)Mossina, Dalmau, and And{\'e}ol]{mossina2024conformal}
Luca Mossina, Joseba Dalmau, and L{\'e}o And{\'e}ol.
\newblock Conformal semantic image segmentation: Post-hoc quantification of predictive uncertainty.
\newblock In \emph{Proceedings of the IEEE/CVF Conference on Computer Vision and Pattern Recognition}, pages 3574--3584, 2024.

\bibitem[Nehme et~al.(2023)Nehme, Yair, and Michaeli]{nehme2023uncertainty}
Elias Nehme, Omer Yair, and Tomer Michaeli.
\newblock Uncertainty quantification via neural posterior principal components.
\newblock \emph{Advances in Neural Information Processing Systems}, 36:\penalty0 37128--37141, 2023.

\bibitem[Qu et~al.(2023)Qu, Zhang, Qiao, Tang, Yuille, Zhou, et~al.]{qu2023abdomenatlas}
Chongyu Qu, Tiezheng Zhang, Hualin Qiao, Yucheng Tang, Alan~L Yuille, Zongwei Zhou, et~al.
\newblock Abdomenatlas-8k: Annotating 8,000 ct volumes for multi-organ segmentation in three weeks.
\newblock \emph{Advances in Neural Information Processing Systems}, 36:\penalty0 36620--36636, 2023.

\bibitem[Ronneberger et~al.(2015)Ronneberger, Fischer, and Brox]{ronneberger2015u}
Olaf Ronneberger, Philipp Fischer, and Thomas Brox.
\newblock U-net: Convolutional networks for biomedical image segmentation.
\newblock In \emph{Medical image computing and computer-assisted intervention--MICCAI 2015: 18th international conference, Munich, Germany, October 5-9, 2015, proceedings, part III 18}, pages 234--241. Springer, 2015.

\bibitem[Salvi et~al.(2025)Salvi, Seoni, Campagner, Gertych, Acharya, Molinari, and Cabitza]{salvi2025explainability}
Massimo Salvi, Silvia Seoni, Andrea Campagner, Arkadiusz Gertych, U~Rajendra Acharya, Filippo Molinari, and Federico Cabitza.
\newblock Explainability and uncertainty: Two sides of the same coin for enhancing the interpretability of deep learning models in healthcare.
\newblock \emph{International Journal of Medical Informatics}, page 105846, 2025.

\bibitem[Teneggi et~al.(2023)Teneggi, Tivnan, Stayman, and Sulam]{teneggi2023trust}
Jacopo Teneggi, Matthew Tivnan, Web Stayman, and Jeremias Sulam.
\newblock How to trust your diffusion model: A convex optimization approach to conformal risk control.
\newblock In \emph{International Conference on Machine Learning}, pages 33940--33960. PMLR, 2023.

\bibitem[Tivnan et~al.(2024)Tivnan, Yoon, Chen, Li, Wu, and Li]{tivnan2024hallucination}
Matthew Tivnan, Siyeop Yoon, Zhennong Chen, Xiang Li, Dufan Wu, and Quanzheng Li.
\newblock Hallucination index: An image quality metric for generative reconstruction models.
\newblock In \emph{International Conference on Medical Image Computing and Computer-Assisted Intervention}, pages 449--458. Springer, 2024.

\bibitem[Van~Aarle et~al.(2015)Van~Aarle, Palenstijn, De~Beenhouwer, Altantzis, Bals, Batenburg, and Sijbers]{van2015astra}
Wim Van~Aarle, Willem~Jan Palenstijn, Jan De~Beenhouwer, Thomas Altantzis, Sara Bals, K~Joost Batenburg, and Jan Sijbers.
\newblock The astra toolbox: A platform for advanced algorithm development in electron tomography.
\newblock \emph{Ultramicroscopy}, 157:\penalty0 35--47, 2015.

\bibitem[Van~Aarle et~al.(2016)Van~Aarle, Palenstijn, Cant, Janssens, Bleichrodt, Dabravolski, De~Beenhouwer, Joost~Batenburg, and Sijbers]{van2016fast}
Wim Van~Aarle, Willem~Jan Palenstijn, Jeroen Cant, Eline Janssens, Folkert Bleichrodt, Andrei Dabravolski, Jan De~Beenhouwer, K~Joost~Batenburg, and Jan Sijbers.
\newblock Fast and flexible x-ray tomography using the astra toolbox.
\newblock \emph{Optics express}, 24\penalty0 (22):\penalty0 25129--25147, 2016.

\bibitem[Wasserthal et~al.(2023)Wasserthal, Breit, Meyer, Pradella, Hinck, Sauter, Heye, Boll, Cyriac, Yang, et~al.]{wasserthal2023totalsegmentator}
Jakob Wasserthal, Hanns-Christian Breit, Manfred~T Meyer, Maurice Pradella, Daniel Hinck, Alexander~W Sauter, Tobias Heye, Daniel~T Boll, Joshy Cyriac, Shan Yang, et~al.
\newblock Totalsegmentator: robust segmentation of 104 anatomic structures in ct images.
\newblock \emph{Radiology: Artificial Intelligence}, 5\penalty0 (5), 2023.

\bibitem[Webber and Reader(2024)]{webber2024diffusion}
George Webber and Andrew~J Reader.
\newblock Diffusion models for medical image reconstruction.
\newblock \emph{BJR| Artificial Intelligence}, 1\penalty0 (1):\penalty0 ubae013, 2024.

\bibitem[Wundram et~al.(2024)Wundram, Fischer, M{\"u}hlebach, Koch, and Baumgartner]{wundram2024conformal}
Anna~M Wundram, Paul Fischer, Michael M{\"u}hlebach, Lisa~M Koch, and Christian~F Baumgartner.
\newblock Conformal performance range prediction for segmentation output quality control.
\newblock In \emph{International Workshop on Uncertainty for Safe Utilization of Machine Learning in Medical Imaging}, pages 81--91. Springer, 2024.

\end{thebibliography}
\end{document}